\documentclass[]{amsart}
\usepackage{amssymb}
\usepackage{graphicx}
\usepackage{xcolor} 
\usepackage{amsmath}
\usepackage{amsthm}
\usepackage{verbatim}
\usepackage{enumitem}

\newtheorem{theorem}{Theorem}

\newtheorem{question}[theorem]{Question}
\newtheorem{lemma}[theorem]{Lemma}
\newtheorem{proposition}[theorem]{Proposition}
\newtheorem{corollary}[theorem]{Corollary}
\newtheorem*{theorem*}{Theorem}
\newtheorem*{lemma*}{Lemma}

\theoremstyle{definition}

\newtheorem{remark}{Remark}

\newtheorem*{remark*}{Remark}
\newtheorem*{example*}{Example}
\newtheorem*{er*}{Examples and Remarks}

\newcommand{\nnrm}[1]{{\vert\kern-0.25ex\vert\kern-0.25ex\vert #1 \vert\kern-0.25ex\vert\kern-0.25ex\vert}}

\begin{document}

\title[Fourier series and deep neural network]
{
Pointwise convergence of Fourier series and deep neural network
for the indicator function of d-dimensional ball
} 

\author{Ryota Kawasumi}
\address{Faculty of Business Administration, Kobe Gakuin University, 1-1-3 Minatojima, Chuo-ku, Kobe, Hyogo, 650-8586, Japan} 
\email{kawasumi@ba.kobegakuin.ac.jp} 

\author{Tsuyoshi Yoneda} 
\address{Graduate School of Economics, Hitotsubashi University, 2-1 Naka, Kunitachi, Tokyo 186-8601, Japan} 
\email{t.yoneda@r.hit-u.ac.jp}


\date{\today} 


\begin{abstract} 
In this paper, we
clarify the crucial difference between
a deep neural network and the Fourier series.
For the multiple Fourier series of periodization of some
radial functions on $\mathbb{R}^d$, Kuratsubo (2010) investigated the behavior of the
spherical partial sum and discovered the third phenomenon other than the well-known Gibbs-Wilbraham and Pinsky phenomena.
In particular, the third one
exhibits prevention of pointwise convergence.
In contrast to it,  we give a specific deep neural network 
and prove  pointwise convergence.
\end{abstract}

\maketitle

\section{Introduction}
In this paper we
clarify the crucial difference between
a deep neural network and the Fourier series.
Let $d\gg1$ be a dimension and fix it.
For $x\in[-1,1)^d$, let  $f^\circ$ be  the indicator function of the $d$ dimensional ball such that 
\begin{equation*}
f^\circ(x)=
\begin{cases}
1,\quad |x|\leq 1/2,\\
0,\quad |x|>1/2.
\end{cases}
\end{equation*}
Let $\rho:\mathbb{Z}_{\geq 1}\to\mathbb{Z}_{\ge1}$ be a monotone increasing function with $\rho(N)\to\infty$ ($N\to\infty$).
Let $f_N:\mathbb{R}^{\rho(N)}\times[-1,1)^d\to\mathbb{R}$ be a prescribed piecewise smooth real-valued function with parameter $N\in\mathbb{Z}_{\geq 1}$,
for example, the sum of $\rho(N)\sim N^d$ terms (i.e. Fourier series) or composite functions $\rho(N)\sim\log N$ times (i.e. deep neural networks).
Also, let $\{W^t\}_{t=0}^\infty$ be sets of parameters, and let $W^t:=\{w^t_j\}_{j=1}^{\rho(N)}\in\mathbb{R}^{\rho(N)}$ be generated by the following gradient descent:
\begin{equation*}
E(W^t):=\frac{1}{2}
\int_{[-1,1)^d}|f_N(W^t,x)-f^\circ(x)|^2dx,
\end{equation*}
\begin{equation}\label{usual grad descent}
w^{t+1}_j=w^t_j-\frac{\partial}{\partial w^t_j} E(W^t).
\end{equation}
Let $\{f^\circ_N\}_{N=1}^\infty$ be a sequence of functions
such that (if there exists)
\begin{equation*}
f_{N}^\circ(x):=\lim_{t\to\infty}f_N(W^t,x)\quad\text{for}\quad x\in [-1,1)^d.
\end{equation*}
Throughout this paper, we consider the following question:
\begin{question}\label{question}
For given $f_N$,  
 can we find a sequences of initial parameters $W^{t=0}_N\in\mathbb{R}^{\rho(N)}$ ($N=1,2,\cdots$)
 assuring the following pointwise convergence? 
\begin{equation*}
\lim_{N\to\infty}f^\circ_N(x)=f^\circ(x)\quad\text{for all}\quad
x\in [-1,1)^d.
\end{equation*}
\end{question}

Let us briefly review recent studies related to this question
(for classical universal approximation theorems, see \cite{C, EMW, H, LL, Steunwart} for example).
Suzuki \cite{S} clarified that deep learning can achieve the minimax optimal rate and outperform any nonadaptive linear estimator such as kernel regression, 
when the target function is in the supercritical Besov spaces $B^s_{p,q}$ with $p<2$ and $s<d/2$ ($d$ is the dimension, note that the case $s=d/2$ is called ``critical"),
which indicates the spatial inhomogeneity of the shape of the target function including the non-smooth functions.
We briefly explain the key idea of \cite{S} in the following:
To show the approximation error theorems,
he first applied the wavelet expansion to the target functions
and then approximated each wavelet bases (composed of spline functions) by ReLU-DNN (see \cite{Y}).
More specifically, the following assertion holds:
\begin{lemma}{\rm \cite[Propositon~2]{Y}}\label{lem:xx}
The function $f(x) = x^2$on the segment $[0,1]$ can be approximated with any
error $\epsilon >0$ by a ReLU network having the depth and the number of weights and computation
units ${\mathcal O}(\log 1/\epsilon)$.
\end{lemma}
We emphasize that the proof of this lemma is based on ``'separability of function spaces''.
By this Lemma~\ref{lem:xx}, for deriving the multi-dimensional polynomials, it suffices to apply the following formula:
\begin{equation*}
xy=\frac{1}{2}((x+y)^2-x^2-y^2),
\end{equation*}
and then we can easily approximate multi-dimensional spline functions
by ReLU-DNN.
The other idea (for variance) is applying the statistical argument in 
\cite{SH}
combined with a covering number evaluation.
To the contrary, since Question \ref{question} is essentially based on ``non-separable function spaces'' as $L^\infty(\mathbb{R}^d)$,
 we cannot  directly apply their argument to ours.

Before going any further, we point out that manipulating the structure of separable
 function spaces may not be enough to capture the discontinuity structure of $f^\circ$. This means that, to  make any further progress, we may need to 
directly look into each DNN dynamics.
The flavor of this insight seems very similar to the recent mathematical studies on the incompressible inviscid flows. See \cite{BL,BL1, EJ, EM, MY} for example.
More precisely, these have been 
directly looking into the dynamics of inviscid fluids in the critical function spaces (to show ill-posedness), and the argument seems quite different from the classical studies focusing on well-posedness in subcritical type of function spaces. See \cite{Chae, Kato, KP, PP, V} for example. To show the well-posedness, the structure of subcritical function spaces, more precisely, 
commutator estimates are crucially indispensable.

This paper is organized as follows: In the next section, we explain the Fourier series case.
In Section 3, we investigate the deep neural network case.
In the last section, we give a proof of the main theorem.

\section{The Fourier series case}

In this section we consider Question \ref{question} in the Fourier series case.
Let $f_N$ be a Fourier series with spherical partial sum:
\begin{equation*}
\begin{split}
f_N(W^t,x)
&:=\sum_{\stackrel{k \in \mathbb{Z}^d,}{|k|< N}}\left(a^t_k\cos (k\pi\cdot x)+b^t_k\sin (k\pi\cdot x)\right),\\
W^t
&:=\{a_k^t\}_{|k|<N}\cup\{b^t_k\}_{|k|<N}
\subset\mathbb{R}.
\end{split}
\end{equation*}
By Parseval's identity, we immediately  have the following convergences:  
For any $W^{t=0}\in\mathbb{R}^{\rho(N)}$, we have 
\begin{equation*}
\lim_{t\to\infty}f_N(W^t,x)=f^\circ_N(x) \quad\text{for}\quad x\in[-1,1)^d,
\end{equation*}
where
\begin{equation*}
\begin{split}
&f^\circ_N(x)
=\sum_{\stackrel{k \in \mathbb{Z}^d,}{|k|< N}}\left(a^\infty_k\cos (k\pi\cdot x)+b^\infty_k\sin (k\pi\cdot x)\right),\\
&a^\infty_k=\int_{[-1,1)^d}f^\circ(x)\cos(k\pi\cdot x)dx\quad\text{and}\quad  b^\infty_k=\int_{[-1,1)^d}f^\circ(x)\sin(k\pi\cdot x)dx.
\end{split}
\end{equation*}
Then we obtain the following proposition.
\begin{proposition}
Let $d\geq 5$.
Then, for any $x\in\mathbb{Q}^d\cap [-1,1)^d$,
\begin{equation*}
f^\circ_N(x)-f^\circ(x) \quad\text{diverges as}\quad N\to\infty.
\end{equation*}
Thus the answer of Question \ref{question} in this case is ``negative".
\end{proposition}
The proof is just direct consequence of Kuratsubo \cite{K} (see also \cite{KNO,KN}). 
Note that, for the Fourier series of the indicator functions
of such several dimensional balls,
the Gibbs and Pinsky phenomena have already been well-known.
The assertion of this proposition is that Kuratsubo \cite{K} discoverd the third phenomenon (preventing pointwise convergence), other than the previous two phenomena.
To understand this third phenomenon intuitivelly, see numerical computations in Section 7 in \cite{KNO}.

\section{The specific deep neural network case}

In this section we investigate the deep neural network case.
Let $N=2^n$ ($n\in\mathbb{N}$) and 
let $h$ be the ReLU function such that 
\begin{equation*}
h(x):=\max\{x,0\}, \quad\text{for}\quad x\in {\mathbb R}.
\end{equation*}
We now construct a deep neural network $f_N$ with a special shape.
For the initial layer, we define
\begin{equation*}
z^1:= z^1(x) :=h(w x+b):=
\begin{pmatrix}
h(w_1\cdot x+b_1)\\
\vdots\\
h(w_{2^n}\cdot x+b_{2^n})
\end{pmatrix}
\end{equation*}
for 
$x\in[-1,1)^d$, $w:=\{w_j\}_{j=1}^{2^n}\in\mathbb{R}^{2^n\times d}$, $b:=\{b_j\}_{j=1}^{2^n}\in\mathbb{R}^{2^n}$ and  $z^1 := \{z_j^1\}_{j=1}^{2^n}\in\mathbb{R}^{2^n}$. 
Moreover we set
\begin{equation}\label{b}
b_j:=\frac{|w_j|}{2}+1.
\end{equation}
For the $2k$-th layer, 
 we put the following sparsity structure:
for $J=1,2,\cdots, 2^{n-k}$ and $1\leq k\leq n$,
\begin{equation}\label{flipping}
\begin{split}
z^{2k}_{3J-2}&=h\left(\frac{1}{2}z^{2k-1}_{2J-1}+\frac{1}{2}z_{2J}^{2k-1}\right)
,\\
z^{2k}_{3J-1}&=h\left(\frac{1}{2}z^{2k-1}_{2J-1}-\frac{1}{2}z_{2J}^{2k-1}\right)
,\\
z^{2k}_{3J}&=h\left(-\frac{1}{2}z^{2k-1}_{2J-1}+\frac{1}{2}z_{2J}^{2k-1}\right),
\\
\end{split}
\end{equation}
where $z^{2k}=\{z_j^{2k}\}_{j=1}^{2^n}$, $z^{2k}\in \mathbb{R}^{3\cdot 2^{n-k}}$.
For the $2k+1$ layer, we put the following linear structure:
for $J=1,2,\cdots 2^{n-k}$, 
\begin{equation}\label{cancellation}
z^{2k+1}_J
=z^{2k}_{3J-2}-z^{2k}_{3J-1}-z^{2k}_{3J}
,
\end{equation}
where $z^{2k+1}\in\mathbb{R}^{2^{n-k}}$.
The key idea of designing such deep neural network is that, combined \eqref{flipping} and \eqref{cancellation} induce cancellations of certain $w_j$, and then  
these cancellations suppress the generation of concavity. Consequently, the convexity remains.

\vspace{0.3cm}

We see that, in the $2n+1$ layer, 
$z^{2n+1}=z^{2n+1}_1$ becomes a real number.
Then we set $f_N$ as the following:
\begin{equation*}
\begin{split}
f_N
&:=\min\{z^{2n+1}_1,1\}.
\end{split}
\end{equation*}
For the initial parameters $W^{t=0}:=\{w^{t=0}_j\}_{j=1}^N$, we assume the following
 largeness condition and 
uniform distribution condition:
\begin{equation}\label{initial condition}
|w^{t=0}_j|
>\frac{1}{\delta}\quad (\text{we clarify this $\delta$ later}).
\\
\end{equation}
For unit vectors: $\tau_j:=\frac{w^t_j}{|w^t_j|} \in\mathbb{S}^{d-1}$ ($j=1,2,\cdots$),
let us define half spaces $H^\circ(\tau_j)$ and the corresponding boundaries
as follows:
\begin{equation*}
\begin{split}
H^\circ_j
&:=\left\{\sigma\in[-1,1)^d: \sigma\cdot\tau_j>-\frac{1}{2}\right\},\\
\partial H^\circ_j
&:=\left\{\sigma\in[-1,1)^d: \sigma\cdot\tau_j=-\frac{1}{2}\right\}.\\
\end{split}
\end{equation*}
Note that this $\tau_j$ is independent of $t$, which will be clarified 
in \eqref{key grad descent}.
Then we impose the following condition to the initial parameters:
\begin{equation}\label{initial condition 2}
\lim_{N\to\infty}\bigcap_{j=1}^N H^\circ_j=\Omega,
\end{equation}
where $\Omega:=\{x\in[-1,1)^d: |x|< 1/2\}$.
The main theorem is as follows:
\begin{theorem}\label{second main theorem}
Assume 
$W^{t=0}\in\mathbb{R}^{N\times d}$ satisfies 
\eqref{initial condition} and 
\eqref{initial condition 2}. 
Then, 
 $f_N(W^t)$ converges to $f^\circ_N$ pointwisely (as $t\to\infty$),
and $f^\circ_N$ converges to $f^\circ$ pointwisely (as $N\to\infty$). 
Moreover we have the following convergence rate:
\begin{equation*}
\|f_N(W^t)-f^\circ_N\|^r_{L^r}\lesssim t^{-1/3}
\quad\text{for}\quad 1\leq r<\infty, 
\end{equation*}
where $a\lesssim b$ means $ a \leq C b$ for some universal constant $C>0$.
Thus the answer of Question \ref{question} in this case is ``positive".
\end{theorem}

\section{Proof of the main theorem.}

Let $x\in[-1,1)^d$. We set   $f_N^\circ(x)$ as follows:
\begin{equation*}
f^\circ_N(x):=
\begin{cases}
1,\quad x\in
\bigcap_{j=1}^{N}H^\circ_j,\\
0,\quad \text{otherwise}.
\end{cases}
\end{equation*}
\begin{remark}\label{covering}
By \eqref{initial condition 2}, we immediately have 
\begin{equation}\label{pointwise estimate}
f_N^\circ(x)\to f^\circ(x)\quad (N\to\infty)\quad\text{for any}\quad x\in[-1,1)^d.
\end{equation}
\end{remark}
Now we calculate the deep neural network.
First we see
\begin{equation*}
\partial_xz^1_j(x)=
\begin{cases}
w_j,\quad x\in D_{0,j},\\
0,\quad x\not\in D_{0,j},
\end{cases}
\end{equation*}
where $D_{0,j}:=\{x:z^1_j(x)>0\}$.
Next we calculate  a pair of $2k-1$, $2k$ and $2k+1$ layers. 
Taking a derivative, we have
\begin{equation*}
\begin{split}
\partial_{z_{2J-1}^{2k-1}}z^{2k+1}_J
=&\frac{1}{2}\partial h\left(\frac{1}{2}z^{2k-1}_{2J-1}+ \frac{1}{2}z^{2k-1}_{2J}\right)\\
&-\frac{1}{2}\partial h\left(\frac{1}{2}z^{2k-1}_{2J-1}-\frac{1}{2}z_{2J}^{2k-1}\right)
+\frac{1}{2}
\partial h\left(-\frac{1}{2}z^{2k-1}_{2J-1}+\frac{1}{2}z^{2k-1}_{2J}\right),
\end{split}
\end{equation*}
where 
\begin{equation*}
\partial h(x)=
\begin{cases}
1\quad x\geq 0,\\
0\quad x<0.
\end{cases}
\end{equation*}
Due to the cancellation of Heaviside functions in the following domain,
\begin{equation}\label{classification 0}
D^0_{k,J}:=\left\{x: 
\frac{1}{2}z^{2k-1}_{2J}(x)>\frac{1}{2}z^{2k-1}_{2J-1}(x)\right\},
\end{equation}
we have 
\begin{equation}
\partial_{z^{2k-1}_{2J-1}}z_J^{2k+1}(x)=1\quad\text{for}\quad x\in D^0_{k,J}.
\label{eq:partial_2J-1_part0}
\end{equation}
Note that, rigorously saying, $z^{2k-1}:=z^{2k-1}\circ z^{2k}\circ\cdots\circ z^1$.
To the contrary, there is no cancellation of Heaviside functions in the following domain: 
\begin{equation}\label{classification 1}
D^1_{k,J}:=\left\{x:
\frac{1}{2} z^{2k-1}_{2J}(x)<\frac{1}{2} z_{2J-1}^{2k-1}(x)\right\}.
\end{equation}
In other words, 
\begin{equation}
\partial_{z^{2k-1}_{2J-1}}z^{2k+1}_J(x)=0\quad\text{for}\quad x\in D^1_{k,J}.
\label{eq:partial_2J-1_part1}
\end{equation}
The same argument goes through also in the case $\partial_{z^{2k-1}_{2J}}z_J^{2k+1}$.
In this case, 
 we have
\begin{align}
\partial_{z^{2k-1}_{2J}}z^{2k+1}_J(x)&=0\quad\text{for}\quad x\in D^0_{k,J}, \label{eq:partial_2J_part0}\\
\partial_{z^{2k-1}_{2J}}z^{2k+1}_J(x)&=1\quad\text{for}\quad x\in D^1_{k,J}.\label{eq:partial_2J_part1}
\end{align} 
We apply these properties inductively in the reverse direction (as the back propergation), and 
we  divide the non-zero region $\{x: f_N(W^t,x)>0\}$ into several parts appropriately.
To do that,
we suitably rewrite 
the natural number $j\in\{1,2,\cdots, 2^{n}\}$
as follows:
\begin{equation*}
j=1+
\sum_{k=1}^{n} 2^{k-1}\delta_n^k(j)
\end{equation*}
where $\delta_n^k(j)\in\{0,1\}$.
 Let
\begin{equation*}
D_j:= D_{0,j}\cap \left(\bigcap_{k=1}^n D_{k,J^k_n(j)}^{\delta_n^k(j)}\right)
\end{equation*}
for
\begin{equation*}
\begin{cases}
J^k_n(j):=1+\sum_{\ell=k+1}^{n}2^{\ell-k-1}\delta^\ell_{n}(j),\quad (k\leq  n-1),\\
J^n_n(j)=1.
\end{cases}
\end{equation*}
For the derivation of this $D_j$, see Appendix \ref{App:Appendix A}.
For example, if $k=n-1$, we see
\begin{equation*}
J_n^{n-1}(j)=1+\delta^n_n(j),
\end{equation*}
if $k=n-2$, we see
\begin{equation*}
J_n^{n-2}(j)=1+\sum_{\ell=n-1}^n2^{\ell-n+1}\delta_n^\ell(j)=1+\delta_n^{n-1}(j)+2\delta_n^n(j).
\end{equation*}
Then we can easily deduce the other cases: $k\leq n-3$.
By using this $D_j$, the derivative formula becomes much simpler:
\begin{equation}\label{spatial derivative}
\partial_xz^{2n+1}(x)=w_j\quad\text{for}\quad x\in 
D_j.
\end{equation}
By the fundamental theorem of calculus,
we have 
\begin{equation*}
z^{2n+1}(x)=\sum_{j=1}^{2^n}\left(h(w_j\cdot x+b_j) \chi_{D_j}(x)
\right).
\end{equation*}
Therefore we obtain the following explicit formula:
\begin{equation}\label{clipped}
f_N(x)=\min\left\{\sum_{j=1}^N\left(h(w_j\cdot x+b_j)
\chi_{D_j}(x)
\right),1\right\}.
\end{equation}
In order to calculate the error function $E(W^t)$, we need to
figure out the rigion $D_j\cap \{x:0<f_N(x)<1\}$ more precisely.
\begin{proposition}\label{front propagation}
We have 
\begin{equation}\label{key estimate}
D_j\cap \{x:0<f_N(x)<1\}=\bigcup_{r\in\left(-\frac{1}{|w^t_j|},0\right)}\mathcal D_j(|w^t_j|, r), 
\end{equation}
where
\begin{equation*}
\begin{split}
\mathcal D_j(|w^t_j|,r)
&:=\partial L_j(r)\setminus \bigcup_{\stackrel{j'=1,2,\cdots, N,}{j'\not=j}} L_{j'}\left(\frac{|w^t_j|r}{|w^t_{j'}|}\right),\\
L_j(r)
&:=\left\{r'\tau_j +x:\ r'\in (-\infty,r]\quad\text{and}\quad x\in \partial H^\circ_j\right\},\\
\partial L_j(r)
&:=\left\{r\tau_j+x:\ x\in \partial H^\circ_j\right\}.
\end{split}
\end{equation*}
Moreover, for any sufficiently small $\epsilon>0$ there exists a $\delta>0$ such that if
$r<\delta$ and  $|w^t_j|^{-1}<\delta$, then
\begin{equation}\label{Jacobian}
\lambda_{d-1}(\mathcal D_j(|w^t_j|, r))>\gamma-\epsilon
\quad\text{with}\quad
\gamma:=\min_{j=1,2,\cdots,N}\lambda_{d-1}(\mathcal D_j(|w^t_j|,0)),
\end{equation}
where $\lambda_d(\cdot)$ is the $d$ dimensional Lebesgue measure.
Note that $\mathcal D_j(|w^t_j|,0)$ is independent of $|w^t_j|$.
\end{proposition}
\begin{proof}
We just 
replace $z^1_j$ to $\partial L_j(r)$ and apply the induction argument again (in the forward direction).
Here we only consider a pair of  $4J$, $4J-1$, $4J-2$ and $4J-3$. 
First we see that
\begin{equation*}
\begin{cases}
\displaystyle\bigcup_{r\in\left(-\frac{1}{|w^t_{4J}|}, 0\right)}\partial L_{4J}(r)= \tilde D_{0,4J},\\
\displaystyle\bigcup_{r\in\left(-\frac{1}{|w^t_{4J}|}, 0\right)}\partial L_{4J-1}\left(\frac{|w^t_{4J}|r}{|w^t_{4J-1}|}\right)= \tilde D_{0,4J-1},\\
\displaystyle\bigcup_{r\in \left(-\frac{1}{|w^t_{4J}|}, 0\right)}\partial L_{4J-2}\left(\frac{|w^t_{4J}|r}{|w^t_{4J-2}|}\right)= \tilde D_{0,4J-2},\\
\displaystyle\bigcup_{r\in \left(-\frac{1}{|w^t_{4J}|}, 0\right)}\partial L_{4J-3}\left(\frac{|w^t_{4J}|r}{|w^t_{4J-3}|}\right)= \tilde D_{0,4J-3},
\end{cases}
\end{equation*}
where $\tilde D_{0, 4J-K}=D_{0,4J-K}\cap\{x:0<z^1_{4J-K}(x)< 1\}$.
 The corresponding functions are $z^1_{4J}$, $z^1_{4J-1}$, $z^1_{4J-2}$ and $z^1_{4J-3}$ respectively.
In fact, for $\sigma\in \partial H^\circ_{4J-K}$ ($K=0,1,2,3,$), we see
\begin{equation*}
\begin{split}
z^1_{4J-K}\left(\frac{|w^t_{4J}|}{|w^t_{4J-K}|}r\tau_{4J-K}+\sigma\right)
&=
h\left(w^t_{4J-K}\cdot\left(\frac{|w^t_{4J}|}{|w^t_{4J-K}|}r\tau_{4J-K}+\sigma\right)+b_{4J-K}\right)\\
&=
h\left(|w^t_{4J}|r-\frac{|w^t_{4J-K}|}{2}+b_{4J-K}\right)\\
&=
h(|w^t_{4J}|r+1).
\end{split}
\end{equation*}
Note that $\tau_j\cdot\tau_j=1$ and
$w^t_j\cdot\sigma=-|w^t_j|/2$ for $\sigma\in\partial H^\circ_j$.
Thus we have 
\begin{equation*}
z^1_{4J-K}\left(\frac{|w^t_{4J}|}{|w^t_{4J-K}|}r\tau_{4J-K}+\sigma\right)\in(0,1)
\quad\text{if and only if}\quad r\in \left(-\frac{1}{|w^t_{4J}|},0\right).
\end{equation*}
By \eqref{classification 0} and \eqref{classification 1}, we see
\begin{equation*}
\begin{cases}
\displaystyle\bigcup_{r\in\left(-\frac{1}{|w^t_{4J}|}, 0\right)}\left(\partial L_{4J}(r)\setminus L_{4J-1}\left(\frac{|w^t_{4J}|r}{|w^t_{4J-1}|}\right)\right)= \tilde D^1_{1,2J},\\
\displaystyle\bigcup_{r\in\left(-\frac{1}{|w^t_{4J}|}, 0\right)}\left(\partial L_{4J-1}\left(\frac{|w^t_{4J}|r}{|w^t_{4J-1}|}\right)\setminus L_{4J}(r)\right)= \tilde D^0_{1,2J},\\
\end{cases}
\end{equation*}
\begin{equation*}
\begin{cases}
\displaystyle\bigcup_{r\in\left(-\frac{1}{|w^t_{4J}|}, 0\right)}
\left(\partial L_{4J-2}\left(\frac{|w^t_{4J}|r}{|w^t_{4J-2}|}\right)\setminus L_{4J-3}\left(\frac{|w^t_{4J}|r}{|w^t_{4J-3}|}\right)\right)= \tilde D^1_{1,2J-1},\\
\displaystyle\bigcup_{r\in\left(-\frac{1}{|w^t_{4J}|}, 0\right)}
\left(\partial L_{4J-3}\left(\frac{|w^t_{4J}|r}{|w^t_{4J-3}|}\right)\setminus L_{4J-2}\left(\frac{|w^t_{4J}|r}{|w^t_{4J-2}|}\right)\right)=\tilde D^0_{1,2J-1},
\end{cases}
\end{equation*}
where $\tilde D^{0\text{ or }1}_{2,2J-K}=D^{0\text{ or }1}_{1,2J-K}\cap \{x: 0<z^3_{2J-K}(x)< 1\}$.
The corresponding functions are unified as $z^3_{2J}$ for the first two cases,
$z^3_{2J-1}$ for the second two cases. 
In fact,
 for $r\in (-1/|w^t_{4J}|,0)$, we see
\begin{equation*}
\begin{split}
&x\in\partial L_{4J}(r)\setminus L_{4J-1}\left(\frac{|w^t_{4J}|}{|w^t_{4J-1}|}r\right)\\
\stackrel{def}{\Longleftrightarrow}\ 
&x\in \partial L_{4J}(r)\quad\text{and}\quad x\in \left(L_{4J-1}\left(\frac{|w^t_{4J}|}{|w^t_{4J-1}|}r\right)\right)^c\\
&(\text{note that}\ 
\left(L_j(r)\right)^c=\left\{r'\tau_j+x:r'\in(r,\infty)\quad\text{and}\quad x\in\partial H^\circ_j\right\}),\\
\stackrel{def}{\Longleftrightarrow}\ 
&
\begin{cases}
x=r\tau_{4J}+\sigma\quad\text{for}\quad \sigma\in\partial H^\circ_{4J}\\
x=r'\tau_{4J-1}+\sigma'\quad\text{for}\quad r'>(|w^t_{4J}|/|w^t_{4J-1}|)r
\quad\text{and}\quad \sigma'\in \partial H^\circ_{4J-1},
\end{cases}
\end{split}
\end{equation*}
and thus we have
\begin{equation*}
\begin{split}
z^1_{4J}(x)=h(w_{4J}x+b_{4J})
&=h(|w^t_{4J}|r+1)\\
&<h(|w^t_{4J-1}|r'+1)=h(w_{4J-1}x+b_{4J-1})=z^1_{4J-1}(x).
\end{split}
\end{equation*}
Hence, $x \in \tilde D^1_{2,2J-1}$.
By using this inequality (and the opposite version), \eqref{flipping} and \eqref{cancellation}, we have $z^3_{2J}(x)\in (0,1)$ if and only if
\begin{equation*}
x\in
 \bigcup_{r\in\left(-\frac{1}{|w^t_{4J}|}, 0\right)}\left(
\partial L_{4J}(r)\setminus L_{4J-1}\left(\frac{|w^t_{4J}|}{|w^t_{4J-1}|}r\right)\right)= \tilde D^1_{1,2J}
\end{equation*}
or
\begin{equation*}
x\in
 \bigcup_{r\in\left(-\frac{1}{|w^t_{4J}|}, 0\right)}\left(
 \partial L_{4J-1}\left(\frac{|w^t_{4J}|}{|w^t_{4J-1}|}r\right)\setminus L_{4J}(r)\right)= \tilde D^0_{1,2J}.
\end{equation*}
The other cases: $\tilde D^1_{1,2J-1}$ and $\tilde D^0_{1,2J-1}$ are similar, thus we omit these cases. 
Again, by  applying \eqref{classification 0} and \eqref{classification 1} with the same argument as the right above, we see (omit variables)
\begin{equation*}
\begin{cases}
\displaystyle\bigcup_{r\in\left(-\frac{1}{|w^t_{4J}|}, 0\right)}\partial L_{4J}\setminus \left(L_{4J-1}
\cup L_{4J-2}
\cup L_{4J-3}
\right)
= \tilde D^1_{2,J},\\
\displaystyle\bigcup_{r\in\left(-\frac{1}{|w^t_{4J}|}, 0\right)}\partial L_{4J-1}
\setminus \left(L_{4J}
\cup L_{4J-2}
\cup L_{4J-3}
\right)
= \tilde D^1_{2,J},\\
\displaystyle\bigcup_{r\in\left(-\frac{1}{|w^t_{4J}|}, 0\right)}\partial  L_{4J-2}
\setminus \left(L_{4J}
\cup L_{4J-1}
\cup L_{4J-3}
\right)
= \tilde D^0_{2,J},\\
\displaystyle\bigcup_{r\in\left(-\frac{1}{|w^t_{4J}|}, 0\right)}
\partial L_{4J-3}
\setminus \left(L_{4J}
\cup L_{4J-1}
\cup L_{4J-2}
\right)
= \tilde D^0_{2,J},\\
\end{cases}
\end{equation*}
where $\tilde D^{0\text{ or }1}_{2,J}=D^{0\text{ or }1}_{2,J}\cap \{x: 0<z^3_{J}(x)<1\}$.
The corresponding functions are unified as $z^3_J$.
Repeating this argument up to $z^{2n+1}_1$, we have \eqref{key estimate}.
To prove \eqref{Jacobian}, 
 we just apply
the continuity argument: For any sufficiently small $\epsilon>0$, there is a $\delta>0$ such that if 
 $r<\delta$ and $|w^t_j|^{-1}<\delta$, we have 
\begin{equation}\label{continuity}
\lambda_{d-1}(\mathcal D_j(|w^t_j|,r))>\gamma-\epsilon\quad\text{for}\quad 
j=1,2,\cdots, N.
\end{equation}
\end{proof}
\begin{corollary}\label{a-priori bound}
We immediately have
\begin{equation*}
\left|\partial_{|w^t_j|}\lambda_{d-1}(\mathcal D_j(|w^t_j|,r))\right|\lesssim \left|\frac{r}{|w^t_{j'}|}\partial_R\lambda_{d-1}(\mathcal D_j^*(R,r))\right|\lesssim |r|\delta\gamma^*,
\end{equation*}
where 
\begin{equation*}
\mathcal{D}^*_j(R,r)
:=\partial L_j(r)\setminus\bigcup_{\stackrel{j'=1,2,\cdots,N,}{j'\not=j}}L_{j'}(R)
\end{equation*}
and $\gamma^*>0$ is a-priori constant (independent of $|w^t_j|$) such that
\begin{equation*}
\gamma^*
:=\sup_{j=1,2,\cdots, N}\sup_{0<r,R<\delta}\left|\frac{d}{dR}\lambda_{d-1}
\left(\mathcal D^*_j(R,r)\right)\right|.\\
\end{equation*}
\end{corollary}

By \eqref{b}, we have that 
\begin{equation}\label{alpha-beta}
\begin{split}
w_j^t\cdot x+b^t_j
=
|w_j^t|s-\frac{|w_j^t|}{2}+ b_j^t=|w_j^t|s+1
\end{split}
\end{equation}
for $x=s\tau_j+\sigma\in D_j$ ($\sigma\in \partial H^\circ_j$) and $t=0$.
Then applying Proposition \ref{front propagation}, we can explicitly write down the error function $E(W^t)$ as follows (for $t=0$):
\begin{equation*}
\begin{split}
E(W^t)
&
=
\frac{1}{2}
\sum_{j=1}^{2^n}
E_j(W^t)+ \left|\bigcap_{j=1}^N H^\circ_j\setminus \Omega\right|\\
&:=
\frac{1}{2
}
\sum_{j=1}^{2^n}
\int_{-\frac{1}{|w^t_j|}}^0\left(|w^{t}_j|s+1\right)^2\lambda_{d-1}(\mathcal D_j(|w^t_j|,s))ds
+ \left|\bigcap_{j=1}^N H^\circ_j\setminus \Omega\right|.
\end{split}
\end{equation*}
Recall that $\Omega=\{x\in[-1,1)^d: |x|<1/2\}$.
We emphasize that $E_j(W^t)$ does not have any rotational component $\frac{w^t_j}{|w^t_j|}$,
namely, for any $j,j'\in \{1,2,\cdots, N\}$, we see that
\begin{equation*}
\nabla_{w_j}E_{j'}(W^t)=\frac{w_j}{|w_j|}\cdot\partial_{|w_j|}E_{j'}(W^t).
\end{equation*}
Applying Corollary \ref{a-priori bound}, we have 
\begin{equation*}
\begin{split}
\partial_{|w^t_j|}E_j(W^t)
&=-(|w^t_j|s+1)^2\lambda_{d-1}(\mathcal D_j(|w^t_j|,s))\bigg|_{s=-\frac{1}{|w^t_j|}}\times\frac{1}{|w^t_j|^2}
\\
&\quad 
+
\int_{-\frac{1}{|w^t_j|}}^0\partial_{|w^t_j|}\left(\left(|w^{t}_j|s+1\right)^2\right)\lambda_{d-1}(\mathcal D_j(|w^t_j|,s))ds\\
&\quad 
+
\int_{-\frac{1}{|w^t_j|}}^0\left(|w^{t}_j|s+1\right)^2\partial_{|w^t_j|}\lambda_{d-1}(\mathcal D_j(|w^t_j|,s))ds\\
&\leq
\int_{-\frac{1}{|w^t_j|}}^02s(|w^t_j|s+1)\lambda_{d-1}(\mathcal D_j(|w^t_j|,s))ds\\
\\&\quad
+\delta\gamma^*\left|\int_{-\frac{1}{|w^t_j|}}^0\left(|w^{t}_j|s+1\right)^2sds\right|\\
&=:(RHS).\\
\end{split}
\end{equation*}
By \eqref{continuity} and change of variable: $r=|w^t_j|s+1$, we see
\begin{equation*}
\begin{split}
&\int_{-\frac{1}{|w^t_j|}}^02s(|w^t_j|s+1)\lambda_{d-1}(\mathcal D_j(|w^t_j|,s))ds\\
=&
\int_{0}^1\frac{2(r-1)r}{|w^t_j|^2}\lambda_{d-1}\left(\mathcal D_j\left(|w^t_j|,\frac{r-1}{|w^t_j|}\right)\right)dr\\
&\lesssim -\frac{\gamma-\epsilon}{|w^t_j|^2}<0.
\end{split}
\end{equation*}
Also by the similar calculation, we see
\begin{equation*}
\left|\int_{-\frac{1}{|w^t_j|}}^0(|w^t_j|r+1)^2rdr\right|
\gtrsim \frac{1}{|w^t_j|^2}.
\end{equation*}
Therefore we have 
\begin{equation*}
(RHS)\lesssim -\frac{\gamma-\epsilon-\delta\gamma^*}{|w^t_j|^2}\lesssim -\frac{1}{|w^t_j|^2}<0.
\end{equation*}
On the other hand, for $j'\not=j$, we see
\begin{equation*}
\partial_{|w^t_{j}|}E_{j'}(W^t)
=
\int_{-\frac{1}{|w^t_{j'}|}}^0\left(|w^{t}_{j'}|s+1\right)^2\partial_{|w^t_j|}\lambda_{d-1}(\mathcal D_{j'}(|w^t_{j}|,s))ds.
\end{equation*}
By the definition,
 we can deduce that 
\begin{equation*}
\mathcal D_{j'}(|w^t_j|+\epsilon,s)\subset \mathcal D_{j'}(|w^t_j|,s)
\end{equation*} 
for any sufficiently small $\epsilon>0$.
 This means that
\begin{equation*}
\partial_{|w_j|}\lambda_{d-1}(\mathcal D_{j'}(|w^t_j|,s))\leq 0\quad \text{and then}\quad 
\partial_{|w^t_j|}E_{j'}(W^t)\leq 0.
\end{equation*}
Thus we have the following simplified version of gradient descent inductively:
\begin{equation}\label{key grad descent}
\begin{split}
|w^{t+1}_j|=
&|w^{t}_j|-\frac{\partial}{\partial |w^{t}_j|} E_j(W^t)-\sum_{\stackrel{j'\in\{1,2,\cdots, N\},}{j'\not=j}}\frac{\partial}{\partial |w_j^t|}E_{j'}(W^t)\gtrsim |w_j^t|
+\frac{1}{|w^t_j|^2}\\
\quad\text{with}\quad
w^{t}_j=
&\frac{w^{t=0}_j}{|w^{t=0}_j|}|w^{t}_j|.
\end{split}
\end{equation}
By directly solving the ODE: $\frac{d}{dt}g(t)=1/g(t)^2$, applying the mean-value theorem and the comparison principle, we have 
$|w^{t}_j|\lesssim t^{1/3}$.
 This guarantees pointwise convergence, and 
 we finally
have the following desired estimate:
\begin{equation*}
\|f_N(W^t)-f^\circ_N\|_{L^r}^r
\lesssim t^{-1/3}.
\end{equation*}


\appendix

\section{}
\label{App:Appendix A}
Here, we derive \eqref{spatial derivative} with precise computation. First, differentiation of $z^{2n+1}_1$ ($J=1$) is expressed in the following equality:
\begin{align*}
\frac{\partial z^{2n+1}_1}{\partial x} &= 
\frac{\partial z^{2n+1}_1}{\partial z_1^{2n-1}} 
\frac{\partial z_1^{2n-1} }{\partial x} 
+
\frac{\partial z^{2n+1}_1}{\partial z_2^{2n-1}} 
\frac{\partial z_2^{2n-1}}{\partial x} 
\\&=
\begin{cases}
\frac{\partial z_1^{2n-1} }{\partial x}  &\text{if}\quad  x \in D_{n,1}^0,\\
\frac{\partial z_2^{2n-1} }{\partial x}  &\text{if}\quad  x \in D_{n,1}^1,
\end{cases}
\end{align*}
where we used equalities \eqref{eq:partial_2J-1_part0}-\eqref{eq:partial_2J_part1} for the second equation.
Next, by applying  differentiation of $z^{2n-1}$, we estimate
\begin{align*}
\frac{\partial z^{2n+1}_1}{\partial x} &=
\begin{cases}
\frac{\partial z_1^{2n-1}}{\partial z_1^{2n-3}} 
\frac{\partial z_1^{2n-3} }{\partial x} 
+
\frac{\partial z_1^{2n-1}}{\partial z_2^{2n-3}} 
\frac{\partial z_2^{2n-3}}{\partial x}  &\text{if}\quad x \in D_{n,1}^0\\
\frac{\partial z_2^{2n-1}}{\partial z_3^{2n-3}} 
\frac{\partial z_3^{2n-3} }{\partial x} 
+
\frac{\partial z_2^{2n-1}}{\partial z_4^{2n-3}} 
\frac{\partial z_4^{2n-3}}{\partial x} &\text{if}\quad x \in D_{n,1}^1
\end{cases}
\\&=
\begin{cases}
\frac{\partial z_1^{2n-3} }{\partial x}  &\text{if}\quad  x \in D_{n,1}^0\cap D_{n-1,1}^0,\\
\frac{\partial z_2^{2n-3} }{\partial x}  &\text{if}\quad  x \in D_{n,1}^0\cap D_{n-1,1}^1,\\
\frac{\partial z_3^{2n-3} }{\partial x}  &\text{if}\quad  x \in D_{n,1}^1\cap D_{n-1,2}^0, \\
\frac{\partial z_4^{2n-3} }{\partial x}  &\text{if}\quad  x \in D_{n,1}^1\cap D_{n-1,2}^1.
\end{cases}
\end{align*}
Again, by applying  differentiation of $z^{2n-2}$, we estimate
\begin{align*}
\frac{\partial z^{2n+1}_1}{\partial x} &=
\begin{cases}
\frac{\partial z_1^{2n-3} }{\partial z^{2n-5}_1}\frac{\partial z^{2n-5}_1}{\partial x}
+
\frac{\partial z_1^{2n-3} }{\partial z^{2n-5}_2}\frac{\partial z^{2n-5}_2}{\partial x}
  &\text{if}\quad  x \in D_{n,1}^0\cap D_{n-1,1}^0\\
\frac{\partial z_2^{2n-3} }{\partial z^{2n-5}_3}
\frac{\partial z^{2n-5}_3}{\partial x}
+\frac{\partial z_2^{2n-3} }{\partial z^{2n-5}_4}
\frac{\partial z^{2n-5}_4}{\partial x}
  &\text{if}\quad  x \in D_{n,1}^0\cap D_{n-1,1}^1\\
\frac{\partial z_3^{2n-3}}{\partial z^{2n-5}_5}\frac{\partial z^{2n-5}_5}{\partial x}+
\frac{\partial z_3^{2n-3}}{\partial z^{2n-5}_6}\frac{\partial z^{2n-5}_6}{\partial x}
 &\text{if}\quad  x \in D_{n,1}^1\cap D_{n-1,2}^0 \\
\frac{\partial z_4^{2n-3} }{\partial z^{2n-5}_7}\frac{\partial z^{2n-5}_7}{\partial x}+  
\frac{\partial z_4^{2n-3} }{\partial z^{2n-5}_8}\frac{\partial z^{2n-5}_8}{\partial x}
&\text{if}\quad  x \in D_{n,1}^1\cap D_{n-1,2}^1
\end{cases}
\\&=
\begin{cases}
\frac{\partial z_1^{2n-5} }{\partial x}  &\text{if}\quad  x \in D_{n,1}^0\cap D_{n-1,1}^0\cap D^0_{n-2,1},\\
\frac{\partial z_2^{2n-5} }{\partial x}  &\text{if}\quad  x \in D_{n,1}^0\cap D_{n-1,1}^0\cap D^1_{n-2,1},\\
\frac{\partial z_3^{2n-5} }{\partial x}  &\text{if}\quad  x \in D_{n,1}^0\cap D_{n-1,1}^1
\cap D^0_{n-2,2}, \\
\frac{\partial z_4^{2n-5} }{\partial x}  &\text{if}\quad  x \in D_{n,1}^0\cap D_{n-1,1}^1
\cap D^1_{n-2,2},\\
\frac{\partial z_5^{2n-5} }{\partial x}  &\text{if}\quad  x \in D_{n,1}^1\cap D_{n-1,2}^0\cap D^0_{n-2,3},\\
\frac{\partial z_6^{2n-5} }{\partial x}  &\text{if}\quad  x \in D_{n,1}^1\cap D_{n-1,2}^0\cap D^1_{n-2,3},\\
\frac{\partial z_7^{2n-5} }{\partial x}  &\text{if}\quad  x \in D_{n,1}^1\cap D_{n-1,2}^1
\cap D^0_{n-2,4}, \\
\frac{\partial z_8^{2n-5} }{\partial x}  &\text{if}\quad  x \in D_{n,1}^1\cap D_{n-1,2}^1
\cap D^1_{n-2,4}.
\end{cases}
\end{align*}
Hence, repeating the same way as the above, we have
\begin{align*}
\frac{\partial z^{2n+1}_1}{\partial x}&=
\begin{cases}
\frac{\partial z_1^1 }{\partial x}  &\text{if}\quad  x \in D_{n,1}^0\cap D^0_{n-1,1}\cap\cdots\cap D_{2,1}^0\cap D_{1,1}^0\\
\frac{\partial z_2^1 }{\partial x}  &\text{if}\quad  x \in D_{n,1}^0\cap D_{n-1,1}^0\cap\cdots \cap D^0_{2,1}\cap D_{1,1}^1\\
\frac{\partial z_3^1 }{\partial x}  &\text{if}\quad  x \in D_{n,1}^0\cap D_{n-1,1}^0\cap\cdots \cap D^1_{2,1}\cap  D_{1,2}^0\\
&\qquad\vdots\\
\frac{\partial z_{2^n}^1 }{\partial x}  &\text{if}\quad  x \in D_{n,1}^1\cap D_{n-1,2}^1\cap\cdots\cap D^1_{2,2^{n-2}}\cap D_{1,2^{n-1}}^1\\
\end{cases}
\\&=
\begin{cases}
w_1 &\text{if}\quad  x \in D_{n,1}^0\cap D_{n-1,1}^0\cap\cdots\cap D^0_{2,1}\cap D_{1,1}^0
\cap D_{0,1},\\
w_2 &\text{if}\quad  x \in D_{n,1}^0\cap D_{n-1,1}^0\cap\cdots\cap D^0_{2,1}\cap D_{1,1}^1
\cap D_{0,2},\\
w_3  &\text{if}\quad  x \in D_{n,1}^0\cap D_{n-1,1}^0\cap\cdots \cap D^1_{2,1}\cap  D_{1,2}^0
\cap D_{0,3},\\
&\qquad\vdots\\
w_{2^n} &\text{if}\quad  x \in D_{n,1}^1\cap D_{n-1,2}^1\cap\cdots\cap D^1_{2,2^{n-2}}\cap D_{1,2^{n-1}}^1
\cap D_{0,2^n}.\\
\end{cases}
\end{align*}
From these formulas with algebraic observation, we can deduce $\delta^k_n(j)$ and $J^k_n(j)$.



\vspace{0.5cm}
\noindent
{\bf Acknowledgments.}\ 
Research of  TY  was partly supported by the JSPS Grants-in-Aid for Scientific
Research 24H00186 and 20H01819.


\vspace{0.5cm}
\noindent
{\bf Conflict of Interest.}\ 
The authors have no conflicts to disclose.
\bibliographystyle{amsplain}


\end{document}